\documentclass{amsart}

\usepackage[foot]{amsaddr}

\usepackage[english]{babel}
\usepackage{amssymb,amsthm,amsfonts}
\usepackage[utf8]{inputenc} 
\usepackage[T1]{fontenc}    
\usepackage{amsmath,amsthm}
\usepackage{graphicx}
\usepackage[colorlinks=true, allcolors=blue]{hyperref}
\usepackage{todonotes}
\usepackage{hyperref}       
\usepackage{url}            
\usepackage{booktabs}       
\usepackage{nicefrac}       
\usepackage{microtype}      
\usepackage{xcolor}         
\usepackage[normalem]{ulem}

\newtheorem{theorem}{Theorem}

\newtheorem{thm2}{Theorem}
\newtheorem{thm3}{Theorem}
\newtheorem{thm}{Theorem}
\newtheorem{definition}[thm]{Definition}

\newtheorem{lemma}[thm2]{Lemma}
\newtheorem{remark}[thm3]{Remark}

\usepackage{tikz}
\usetikzlibrary{calc, intersections}
\usetikzlibrary{arrows, shapes, fit}
\usetikzlibrary{positioning}
\usetikzlibrary{arrows.meta}
\usetikzlibrary{matrix,decorations.pathreplacing, calc, positioning,fit}
\usepackage{tkz-graph}
\usetikzlibrary{patterns,snakes}
\usetikzlibrary{arrows, shapes, fit}
\usetikzlibrary{shapes.geometric,calc}
\usepackage{array}
\usetikzlibrary{arrows.meta}
\tikzset{main node/.style={rectangle,fill=blue!20,draw=none,minimum size=0.3cm,inner sep=3pt, rounded corners = 3pt,font=\sffamily},}

\newcommand{\R}{\mathbb{R}}

\begin{document}
	\title[On the power of graph neural networks]{On the power of graph neural networks  and the role of the activation function}

	\author{Sammy Khalife}
	\address{Johns Hopkins University, Department of Applied Mathematics and Statistics}
	\curraddr{}
	\email{khalife.sammmy@jhu.edu}

	\author{Amitabh Basu}
	\address{}
	\curraddr{}
	\email{basu.amitabh@jhu.edu}
	\maketitle

	\begin{abstract}
		In this article we present new results about the expressivity of Graph Neural Networks (GNNs). 
		We prove that for any GNN with piecewise polynomial activations, whose architecture size does not grow with the graph input sizes, there exists a pair of non-isomorphic rooted trees of depth two such that the GNN cannot distinguish their root vertex up to an arbitrary number of iterations. In contrast, it was already known that unbounded GNNs (those whose size is allowed to change with the graph sizes) with piecewise polynomial activations can distinguish these vertices in only two iterations. It was also known prior to our work that with ReLU (piecewise linear) activations, bounded GNNs  are weaker than unbounded GNNs~\cite{aamand2022exponentially}. Our approach adds to this result by extending it to handle any piecewise polynomial activation function, which goes towards answering an open question formulated by Grohe~\cite{grohe2021logic} more completely.
		Our second result states that if one allows activations that are not piecewise polynomial, then in two iterations a single neuron perceptron can distinguish the root vertices of any pair of nonisomorphic trees of depth two (our results hold for activations like the sigmoid, hyperbolic tan and others). This shows how the power of graph neural networks can change drastically if one changes the activation function of the neural networks.  The proof of this result utilizes the Lindemann-Weierstrauss theorem from transcendental number theory.
	\end{abstract}
	
	\section{Introduction}
	
	Graph Neural Networks (GNNs) form a popular framework for a variety of computational tasks involving network data, with applications ranging from analysis of social networks, structure and functionality of molecules in chemistry and biological applications, computational linguistics, simulations of physical systems, techniques to enhance optimization algorithms, to name a few. The interested reader can look at~\cite{zitnik2018modeling,bronstein2017geometric,ding2019cognitive,defferrard2016convolutional,scarselli2008graph, hamilton2020graph,duvenaud2015convolutional,khalil2017learning,stokes2020deep,zhou2020graph,battaglia2016interaction,cappart2021combinatorial,sanchez2020learning}, which is a small sample of a large and actively growing body of work.

	Given the rise in importance of inference and learning problems involving graphs and the use of GNNs for these tasks, significant progress has been made in recent years to understand their computational capabilities. See the excellent recent survey~\cite{jegelka2022theory} for an exposition of some aspects of this research. One direction of investigation is on their so-called {\em separation power} which is the ability of GNNs to distinguish graphs with different structures. In this context, it becomes natural to compare their separation power to other standard computation models on graphs, such as different variants of the {\em Wesfeiler-Lehman algorithm} \cite{cai1992optimal,xu2018powerful,huang2021short}, and the very closely related {\em color refinement algorithm} \cite{grohe2021logic}.  These investigations are naturally connected with {\em descriptive complexity theory}, especially to characterizations in terms of certain logics; see~\cite{grohe2017descriptive,grohe2021logic} for introductions to these different connections. A closely related line of work is to investigate how well general functions on the space of graphs can be approximated using functions represented by GNNs; see~\cite{azizian2021expressive,keriven2019universal,maron2019universality,geerts2022expressiveness} for a sample of work along these lines. Our work in this paper focuses on improving our understanding of the separation power of GNNs.
	
	At a high level, the computational models of GNNs, Wesfeiler-Lehman/color refinement type algorithms and certain logics in descriptive complexity are intimately connected because they all fall under the paradigm of trying to discern something about the global structure of a graph from local neighborhood computations. Informally, these algorithms iteratively maintain a state (a.k.a. ``color'') for each vertex of the graph and in every iteration, the state of a vertex is updated by performing some predetermined set of operations on the set of current states of its neighbors (including itself). The different kinds of allowed states and allowed operations determine the computational paradigm. For instance, in GNNs, the states are typically vectors in some Euclidean space and the operations for updating the state are functions that can be represented by deep neural networks. As another example, in the color refinement algorithm, the states are multisets of some predetermined finite class of labels and the operations are set operations on these multisets. A natural question then arises: Given two of these models, which one is more powerful, or equivalently, can one of the models always simulate the other? Several mathematically precise answers to such questions have already been obtained. For instance, it has been proved independently by~\cite{morris2019weisfeiler} and~\cite{xu2018powerful} that the color refinement algorithm precisely captures the expressiveness of GNNs in the sense that there is a GNN distinguishing two nodes of a graph (by assigning them different state vectors) if and only if color refinement assigns different multisets to these nodes. Such a characterization holds for \emph{unbounded} GNNs, i.e. GNNs for which the underlying neural network sizes can grow with the size of the input graphs. This implies a characterization of the distinguishability of nodes by GNNs as being equivalent to what is known as \emph{Graded Modal Counting Logic} (GC2); see \cite{barcelo2020logical} for some recent, quantitatively precise results in this direction.

	Reviewing these equivalences in a recent survey~\cite{grohe2021logic}, Grohe emphasizes the fact that the above mentioned equivalence between the separation power of GNNs and the color refinement algorithm has only been established for unbounded GNNs whose neural network sizes are allowed to grow as a function of the size of the input graphs. Question 1 on his list of important open questions in this topic asks what happens if one considers {\em bounded} GNNs, i.e., the the size of the neural networks is fixed a priori and cannot change as a function of the size of the input graphs. {\em Do bounded GNNs have the same separation power as unbounded GNNs and color refinement?} 
	{In \cite{aamand2022exponentially}, the authors  provide a negative answer to this question in the case of Rectified Linear Unit (ReLU) activations, by proving lower bounds on the size of the GNNs with ReLU activations that are needed to simulate color refinement. Their approach combines results from communication complexity and properties of ReLU neural networks, including upper bounds on the number of their linear regions. In this article, we answer this question in the negative in the more general case of \emph{piecewise polynomial activations}}. Given any bounded GNN with such activations, we explicitly construct two non isomorphic rooted trees of depth two such that their root nodes cannot be distinguished by the GNN. Interestingly, only the sizes of the trees depend on the GNN, but their depth does not. This result is stated formally in Theorem \ref{theorem:uniform} and it holds for bounded GNNs with piecewise polynomial activations (this includes, e.g., ReLU activations). We prove a second result that shows how the activation function dramatically impacts the expressivity of bounded size GNNs: if one allows activation functions that are not piecewise polynomial, all root nodes of rooted trees of depth two can be distinguished by a single neuron perceptron. This result is formally stated in Theorem \ref{result:other:activations}. \cite{amir2024neural} provide a more general result for analytic non polynomial activation functions: it is possible to simulate color refinement \emph{on any graph} if one allows analytic non polynomial functions. \cite{bravo2024dimensionality} also showed that  among GNNs with analytic non polynomial activations, those with one dimensional embeddings at every iteration 
	are sufficient to simulate color-refinement. In contrast with approaches presented in \cite{amir2024neural,bravo2024dimensionality}, our proof does not use the analytic properties of the activation function, but linear independence over the integers of the activation function evaluated at distinct integers.

	The rest of this article is organized as follows. In Section \ref{sec:2} we present the main definitions and formal statement of our results. In Section \ref{sec:3} we give an overview of the proofs. Sections \ref{sec:4} and  \ref{sec:5} fill in the technical details. 
	
	\section{Formal statement of results}\label{sec:2}

	We assume graphs to be finite, undirected, simple, and vertex-labelled: 
	a graph is a tuple $G = (V(G),E(G),P_1(G),...,P_{\ell}(G)) $ consisting of a finite vertex set $V(G)$, a binary edge relation $E(G) \subset V(G)^{2}$ that is symmetric and irreflexive, and unary relations $P_1(G),\cdots ,P_{\ell}(G) \subset V(G)$ representing $\ell > 0$ vertex labels. 
	In the following, we suppose that the $P_i(G)$'s form a partition of the set of vertices of $G$, i.e. each vertex has a unique label. Also, the number $\ell$ of labels, which we will also call {\em colors}, is supposed to be fixed and does not grow with the size of the input graphs. When there is no ambiguity about which graph $G$ is being considered, $N(v)$ refers to the set of neighbors of $v$ in $G$ not including $v$. $\lvert G \rvert$ will denote the number of vertices of $G$. We use simple curly brackets for a set $X=\{x \in X\}$ and double curly brackets for a multiset $Y=\{\{y \in Y \}\}$. For a set $X$, $\lvert X\rvert $ is the cardinal of $X$. When $m$ is a positive integer, $\mathfrak{S}_m$ is the set of permutations of $\{1, \cdots, m\}$. We will also need the following basic notions from algebra. Given any ring $(R, +, \times)$, $R[X_1, \ldots, X_m]$ will denote the polynomial ring in $m$ indeterminates with coefficients in $R$. 
	An (additive) subgroup generated by $\sigma_1, \cdots, \sigma_q \in R$ is the smallest (inclusion-wise) additive subgroup of $R$ that contains $\sigma_1, \cdots, \sigma_q$. Equivalently, it corresponds to the set $\{ \lambda_1 \sigma_1 + \cdots + \lambda_q \sigma_q : \lambda \in \mathbb{Z}^{q} \}$.

	\begin{definition}[Piecewise polynomial]\label{def:piecewise-poly}
		Let $m$ be a positive integer. A function $f: \mathbb{R}^{m} \rightarrow \mathbb{R}$ is piecewise polynomial iff there exist multivariate polynomials $P_1, \cdots, P_r \in \mathbb{R}[X_1, \cdots, X_m]$ such that for any  $x \in \mathbb{R}^m$, there exists $i\in\{1, \cdots, r\}$ such that $f(x)=P_i(x)$. A  \textit{polynomial region} $A$ of $f$ is a set such that there exists $i\in [r] $ such that $ A \subseteq \{ x \in \mathbb{R}^{m}: f(x) = P_i(x)\}$. The degree of a piecewise polynomial function $f$ is $\mathsf{deg}(f):=\max\{\mathsf{deg}(P_1), \cdots, \mathsf{deg}(P_r)\}$. The {\em number of polynomial pieces} of a piecewise polynomial $f$ is the smallest $r$ such that $f$ can be represented as above.
	\end{definition}
	
	\begin{definition}[Finitely generated polynomial]\label{deffintelygeneratedp}
		Let $q$ be a positive integer. A multivariate polynomial $P$ is $q$-generated provided there exist reals $\sigma_1, \cdots, \sigma_q$ such that $P = \sum_{\alpha \in S } \gamma_{\alpha} X^{\alpha}$, where $\gamma_{\alpha} $ belongs to the additive subgroup of $\mathbb{R}$ generated by $\sigma_1, \cdots, \sigma_q$, 
		and $S$ is the set of exponents of $P$.  Under these conditions, we say that $P$ is $q$-generated by the reals $\sigma_1, \cdots, \sigma_q$.
	\end{definition}

	\begin{definition}[Embedding, equivariance, and refinement]
		Given a set $X$, an embedding $\xi$  is a function that takes as input a graph $G$ and a vertex $v\in V(G)$, and returns an element $\xi(G,v)\in X$. An embedding is \emph{equivariant} if and only if for any pair of isomorphic graphs $G, G'$, and any isomorphism $f$  from G to $G'$, it holds that
		$\xi(G,v) = \xi(G',f(v))$. We say that $\xi$ refines  $\xi'$ if and only if for any graph $G$ and any $v \in V(G), \xi(G, v) = \xi(G, v') \implies \xi'(G, v) = \xi'(G,v')$.
		When the graph $G$ is clear from context, we use $\xi(v)$ as shorthand for $\xi(G,v)$.
	\end{definition}

	\begin{definition}[Color refinement]  Given a graph $G$, and $v \in V(G)$, let  $(G, v) \mapsto \mathsf{col(G,v)}$ be the function which returns the color of the node $v $. The color refinement refers to a procedure that returns a sequence of equivariant embeddings $cr^{t}$, computed recursively as follows:
		
		- $cr^{0}(G,v) = \mathsf{col}(G,v)$ 
		
		-  For $t\geq 0$, $ \mathsf{cr}^{t+1}(G,v) := ( \mathsf{cr}^{t}(G,v), \{\{  \mathsf{cr}^{t}(G,w): w \in N(v) \}\}) $
		
		In each round, the algorithm computes a coloring that is finer than the one computed in the previous round, that is, $\mathsf{cr}^{t+1}$ refines $ \mathsf{cr}^{t}$. 
		For some $t \leq n := \lvert G\rvert $, this procedure stabilises: the coloring does not become strictly finer anymore.
	\end{definition}
	
	\begin{remark}
		We refer the reader to the seminal work \cite[Sections 2 and 5]{cai1992optimal} for comments about the history and connections between the color refinement and Weisfeiler-Lehman algorithms.
	\end{remark}
	
	
	\begin{definition}[Graph Neural Network (GNN)]
		A GNN is a recursive embedding of vertices of a labelled graph by relying on the underlying adjacency information and node features. 
		Each vertex $v$ is attributed an indicator vector $\xi^{0}(v)$ of size $\ell$, encoding the color of the node $v$: the colors being indexed by the palette $\{1, \cdots, \ell\}$, $\xi^{0}(v)=e_i$ (the $i$-th canonical vector) if the color of the vertex $v$ is $i$. The GNN is fully characterized by: 
		
		
		$\circ $ A combination function $\mathsf{comb}: \mathbb{R}^{2\ell}\longrightarrow \mathbb{R}^{\ell}$ which is a feedforward neural network with given activation function $\sigma:\mathbb{R}\longrightarrow \mathbb{R}$.

		$\circ$ The update rule of the GNN at iteration $t \in \mathbb{N}$ for any labelled graph $G$ and vertex $v \in V(G)$, is given as follows:
		$$\xi^{0}(v) \text{ is the indicator vector of the color of $v$}, \quad \xi^{t+1}(v) = \mathsf{comb} (\xi^{t}(v), \sum_{w \in N(v)}\xi^{t}(w) ) $$
	\end{definition}

	\begin{remark} This type of GNN is sometimes referred to as a \textbf{recurrent  GNN}. The general definition (cf. for instance \cite{grohe2021logic}) usually considers a sequence of combine and aggregation  functions which may depend on the iteration $t$. The aggregation functions replace the sum over the neighborhood, i.e. at each iteration $\mathsf{comb}(\xi^{t}(v), \mathsf{agg}(\{ \{ \xi^{t}(w) : w  \in  N (v)  \} \}))$ is the new embedding for vertex $v$. It has been proved in \cite{xu2018powerful} that for any  $\mathsf{agg}$ function, there is a GNN (of potentially larger size) whose aggregation function is  the summation and which refines any GNN with this aggregation function. The results of this article extend to GNNs whose combination and aggregation functions are allowed to be different in different iterations, but are  multivariate piecewise polynomials. For ease of presentation, we restrict to 
		recurrent GNNs. 
	\end{remark}

	Given these definitions, we can now formally state the previously known results about the expressivity of unbounded GNNs (Theorems \ref{theorem:refinement} and \ref{theorem:1}). Namely, in Theorem \ref{theorem:1}, the size of the GNN is allowed to grow with $n$.

	\begin{theorem}\label{theorem:refinement}\cite{grohe2021logic,xu2018powerful,morris2019weisfeiler}
		Let $d \geq 1$, and let $\xi^{d}$ be an {embedding } computed by a GNN after $d$ iterations. Then $\mathsf{cr}^{d}$ refines $\xi$, that is, for all graphs $G, G'$ and vertices $v \in V(G)$, $v' \in V(G')$, $\mathsf{cr}^{d}(v)=\mathsf{cr}^{d}(v') \implies \xi^{d}(G,v)=\xi^{d}(G',v')$.
	\end{theorem}
	\begin{theorem}\cite[Theorem VIII.4]{grohe2021logic}\cite{xu2018powerful,morris2019weisfeiler}\label{theorem:1} Let $n \in \mathbb{N}$. Then there is a recurrent GNN such that for all $t=0, \cdots, n$, the vertex invariant $\xi^{t}$ computed in the $t$-th iteration of the GNN refines $\mathsf{cr}^{t}$ on all graphs of order at most $n$.
	\end{theorem}
	In contrast, we prove Theorems \ref{theorem:uniform} and \ref{result:other:activations} for bounded GNNs:
	\begin{theorem}\label{theorem:uniform} For any GNN, i.e., choice of combination function, represented by a feedforward neural network with piecewise polynomial activation, and any natural number $I \in \mathbb{N}$, there exists a pair of rooted trees $T$ and $T'$ (unicolored, i.e. $\ell=1$) of depth two with root nodes $s$ and $s'$ respectively such that:
		\begin{itemize}
			\item $\mathsf{cr}^{2}(T,s) \neq \mathsf{cr}^{2}(T',s')$, i.e. $s$ and $s'$ can be distinguished with color refinement in two iterations.
			\item $\xi^{t}(T,s) = \xi^{t}(T',s')$ for all $t \leq I$, i.e., $s$ and $s'$ cannot be distinguished by the GNN until iteration $I+1$.
		\end{itemize} 
	\end{theorem}

	\begin{theorem}\label{result:other:activations}
		In two iterations, a  single neuron perceptron with an activation $\sigma \in \{ \exp, \text{\normalfont sigmoid}, \cosh, \sinh, \tanh \}$ can distinguish the root nodes of any pair of non-isomorphic rooted trees of depth two.\end{theorem}

	\section{Overview of the proofs}\label{sec:3}
	
	To establish our first result, we will use rooted trees of the form shown in Figure~\ref{figpolytreecounter} which is a tree of depth two whose depth one vertices have prescribed degrees $k_1, \ldots, k_m$, with $k_1, \ldots, k_m \geq 1$. Given a GNN with piecewise polynomial activation and a natural number $I\in \mathbb{N}$, we will show that there exist two sets of integers $k_1, \cdots, k_m$ and $k'_1, \cdots, k'_m$ that are not the same up to permutations, such that for the corresponding rooted trees $T[k_1, \cdots, k_m]$ and $T[k'_1,\cdots,k'_m]$, the GNN cannot distinguish $s$ and $s'$ for the first $I$ iterations, i.e. $\xi^{t}(T,s) = \xi^{t}(T',s')$ for any $t \in \{1, \cdots, I\}$. 
	Note that the natural numbers $m$, and $k_1, \cdots, k_m$ and $k'_1, \cdots, k'_m$ will depend on $I$, the activation and the size of the neural network considered.
	
	The proof of the first result is structured as follows. Since the trees are parameterized by $m$-tuples of integers $k_1, \ldots, k_m$, the embedding of the root node computed by the GNN at any iteration is a function of these $m$ integers. Since the activations are piecewise polynomials, these embeddings of the root node are also piecewise multivariate polynomial functions of $k_1, \ldots, k_m$. We further prove that the \emph{complexity} of its polynomial pieces can be controlled uniformly in the sense of Definition \ref{deffintelygeneratedp}: the number of generators of such polynomial is independent of $m$, but depends only on the number of iterations of the GNN and the underlying neural network (Lemma \ref{compositiongenerationpol} and \ref{lemma:polyfunction}). Then, we show that there exists a large enough region of $\mathbb{R}^m$ on which this piecewise polynomial function is evaluated by the {\em same} polynomial. This region is large enough in the following sense: we prove that for a dimension $m $ that is sufficiently large, the region contains more integral vectors than the number of possible values a $q$-generated polynomial with degree at most $q$ can take on these vectors, even after identifying vectors up to permutations of the coordinates (Lemmas \ref{lemmanbvaluesgeneratedp} and \ref{lemma:polycollision}).
	This implies that the polynomial piece will take the same value on two distinct integral vectors whose coordinates are not identical up to permutations. 
	When translating this result back to the world of GNNs, this means that the two embeddings of the root nodes of the trees corresponding to these two vectors will coincide. To conclude a separation between bounded and unbounded GNNs, we justify that the unbounded ones can seperate these two vertices. This is based on the previous result (Theorem \ref{theorem:1}) stating that unbounded GNNs refine color refinement.
	
	Our second result states that for activations that are not piecewise polynomial, a one neuron perceptron GNN can distinguish the root nodes of any pair of nonisomorphic trees of depth two. 
	In particular, we prove this when the activation function is the exponential, the sigmoid or the hyperbolic sine, cosine or tangent functions. This is done by showing that the condition $\xi^{2}(s)=\xi^{2}(s)$ corresponds to a relation between the exponentials of the integers $k_1, \cdots, k_m$ and $k'_1, \cdots, k'_m$. Applying the Lindemann-Weirstrass Theorem in transcendental number theory (Lemma \ref{lemmasumexp} and Theorem \ref{lindermannweierstrassth}) leads to the conclusion that $k'_1, \ldots, k'_m$ must be a permutation of $k_1, \ldots, k_
	m$, showing that the trees are isomorphic.

	\begin{figure}[h]
		\centering
		\begin{tikzpicture}
	\begin{scope}[xshift=4cm]
	\node[main node] (1) {$s$};
	\node[main node] (2) [below left = 1.5cm and 4cm  of 1] {$x_1$};
	\node[main node] (3) [right = 2cm  of 2] {$x_2$};
	\node[main node] (4) [below left = 1.5cm  of 2] {};
 \node[] (9) [right = 0.5cm  of 4] {};
 \node[main node] (11) [right = 1cm  of 9] {};
 
 \node[main node] (5) [right = 2.5cm  of 4] {};
 \node[] (6) [right = 0.5cm  of 5] {};
 \node[] (7) [right = 0.5cm  of 6] {};
 \node[main node] (8) [right = 0.5cm  of 7] {};
    \draw[-] (1) edge node[right] {} (3);
	\draw[-] (2) edge node[right] {} (1);
	\draw[-] (2) edge node[right] {} (4);
    \draw[-] (3) edge node[right] {} (5);
    \draw[densely dotted] (3) edge node[right] {} (6);
    \draw[densely dotted] (3) edge node[right] {} (7);
    \draw[-] (3) edge node[right] {} (8);
    \draw[-] (2) edge node[right] {} (11);
    \draw[densely dotted] (2) edge node[right] {} (11);
    \draw[densely dotted] (2) edge node[right] {} (9);



 \node[] (10) [right = 3cm  of 3] {};
	
 \node[] (12) [right= 1.5cm  of 7] {};
 \node[] (13) [right = 0.5cm  of 12] {};
 \node[] (14) [right = 0.5cm  of 13] {};
 \node[] (15) [right = 0.5cm  of 14] {};

 \node[main node] (16) [right = 3cm  of 10] {$x_m$};
 
\node[main node] (17) [right = 1cm  of 15] {};
 \node[main node] (18) [right = 0.5cm  of 17] {};
 \node[main node] (19) [right = 0.5cm  of 18] {};
 \node[main node] (20) [right = 0.5cm  of 19] {};

\draw[densely dotted] (10) edge node[right] {} (1);
\draw[densely dotted] (10) edge node[right] {} (12);
\draw[densely dotted] (10) edge node[right] {} (13);
\draw[densely dotted] (10) edge node[right] {} (14);
\draw[densely dotted] (10) edge node[right] {} (15);

\draw[-] (1) edge node[right] {} (16);
\draw[-] (16) edge node[right] {} (17);
\draw[densely dotted] (16) edge node[right] {} (18);
\draw[densely dotted] (16) edge node[right] {} (19);
\draw[-] (16) edge node[right] {} (20);




	 \draw [
	thick,
	decoration={
		brace,
		mirror,
		raise=0.5cm
	},
	decorate
	] (4.west) -- (11.east) 
	node [pos=0.5,anchor=north,yshift=-0.55cm] {$k_1-1$ vertices};

	\draw [
	thick,
	decoration={
		brace,
		mirror,
		raise=0.5cm
	},
	decorate
	] (5.west) -- (8.east) 
	node [pos=0.5,anchor=north,yshift=-0.55cm] {$k_2-1$ vertices};

\draw [
	thick,
	decoration={
		brace,
		mirror,
		raise=0.5cm
	},
	decorate
	] (12.west) -- (15.east) 
	node [pos=0.5,anchor=north,yshift=-0.55cm] {$\cdots$};

 \draw [
	thick,
	decoration={
		brace,
		mirror,
		raise=0.5cm
	},
	decorate
	] (17.west) -- (20.east) 
	node [pos=0.5,anchor=north,yshift=-0.55cm] {$k_m-1$ vertices}; 


	\end{scope}
	
	\end{tikzpicture}
 
   
		\caption{  $T[k_1, \cdots, k_m]$}
		\label{figpolytreecounter}
	\end{figure}
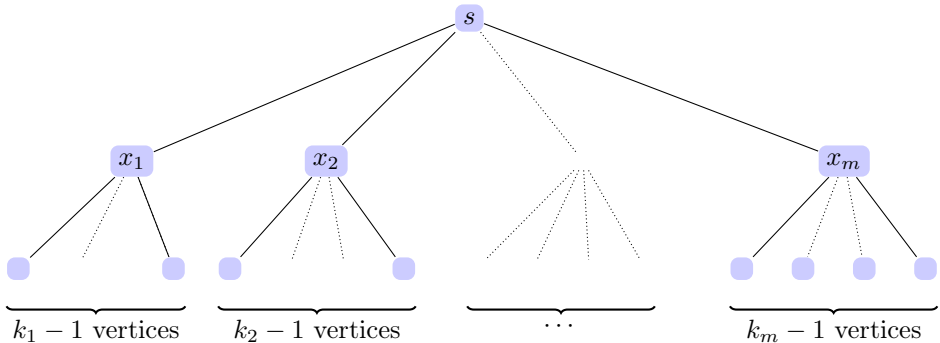

	\section{Collision with piecewise polynomial activations}\label{sec:4}

	\begin{lemma}\label{lemmanbvaluesgeneratedp}
		Let $m$, $q$ and $T$ be positive integers. For any natural number $M$, let $F_M$ be the box $\{ (k_1, \cdots, k_m) \in \mathbb{Z}^{m}: \forall i \quad 1 \leq k_i \leq M\}$ and let $\sigma_1, \cdots, \sigma_q$ be $q$ reals. Let $P_1, \cdots, P_T$ be $T$ polynomials of total degree less than $q$, and $q$-generated by  $\sigma_1, \cdots, \sigma_q$. 
		Then the number of values taken by the function $ x \mapsto (P_1(x), \cdots, P_T(x)) $ on $F_M$ is at most 
		$$\left( 2 \left(\max_{i \in [q], j \in [T],\alpha \in S} \lvert \lambda^j_{\alpha, i} \rvert\right)  M^q \binom{m+q-1}{q} +1\right)^{qT} $$
		where $S$ is the union of the set of exponents of the $P_j$'s and the $\lambda^{j}_{\alpha,i }$ are the coefficients in the decomposition of their coefficients over the generators $\sigma_1, \cdots, \sigma_q$.
		
	\end{lemma}
	\begin{proof}
		Note that if every $P_j$ is a polynomial with integer coefficients, i.e., each $P_j$ is generated by $1$, and has degree less than $q$, then the proof follows by considering the maximum and the minimum of the $P_j$'s over $F_M$, as each of them takes only integer values. To deal with the finitely generated case, we  reduce it to the integer case. If $S$ is the union of the set of exponents of $P_1, \ldots, P_T$, by assumption there exist integers $(\lambda^j_{\alpha, i})_{(\alpha,i) \in S \times \{1, \cdots, q\}}$ such that 
		\begin{align*}
		P_j =  \sum_{\alpha \in S} \gamma_{\alpha,j} X^{\alpha}   = & \sum_{\alpha \in S } \left( \sum_{i=1}^{q} \lambda^j_{\alpha, i} \sigma_{i}\right) X^{\alpha}  \\
		=& \sum_{i=1}^{q} \left( \sum_{\alpha \in S} \lambda^j_{\alpha,i}\sigma_i X^{\alpha} \right)   \\
		=& \sum_{i=1}^{q} \left( \underbrace{\sum_{\alpha \in S} \lambda^{j}_{\alpha,i} X^{\alpha}}_{:=f_{i,j}}  \right) \sigma_i 
		\end{align*}
		
		$f_{i,j}$ is a polynomial of degree less than $q$ with integer coefficients. As we wish to upper-bound the number of values of $\phi: x\mapsto (P_1(x), \cdots, P_T(x))$ on $F_M$, we consider
		\begin{align*}
		\Phi:& \quad F_M \longrightarrow \mathbb{Z}^{q T} \\
		&  \quad  x \mapsto \, (\,  f_{i,j}(x)\,)_{i \in [q], j \in [T]}
		\end{align*} 
		The number of values of $\Phi$ on $F_M$ controls the number of values of $\phi$ on $F_M$ as each of the coordinates of $\phi$ is a linear combination of the coordinates of $\Phi$ with fixed linear coefficients $\sigma_1, \ldots, \sigma_q$. We use the fact that each coordinate of $\Phi(x)$ is a multivariate polynomial with \emph{integer} coefficients: 
		
		\begin{align*}\rvert \phi(F_M) \rvert  \leq  \lvert \Phi (F_M) \rvert  & \leq   \prod_{i \in [q], j \in [T]}  ( \lvert \max_{x \in F_M }f_{i,j}(x)  -\min_{x \in F_M} f_{i,j}(x)\rvert +1)  \\ 
		& \leq \prod_{i \in [q], j \in [T]}  ( 2  \max_{x\in F_M} \lvert f_{i,j}(x) \rvert +1)  \\ 
		& \leq \left( 2 \left(\max_{i \in [q], j \in [T],\alpha \in S} \lvert \lambda^j_{\alpha, i} \rvert\right) M^q \binom{m+q-1}{q} +1\right)^{qT} 
		\end{align*}
		where the last inequality follows from the fact that $\binom{m+q-1}{q}$ is an upper-bound on the number of monomials of $f_{i,j}$.
	\end{proof}

	\begin{lemma}\label{lemma:polycollision}
		{Let $q$ be a positive integer, let $I$ be a finite subset of $ \mathbb{N}$, and let  $(f_{t,m}: \mathbb{R}^{m} \rightarrow \mathbb{R} )_{(t,m)\in I\times \mathbb{N}} $ be a double sequence of piecewise polynomial functions satisfying:
			
			\begin{itemize}
				\item[i)] $\mathsf{deg}(f_{t,m}) \leq q$ for all $(t,m)\in I\times \mathbb{N}$ (bounded degree condition). 
				\item[ii)] for any  $m \in \mathbb{N}$ and $t\in I$, $f_{t,m}$ is $q$-generated.
			\end{itemize}
			Then, there exists $m\in \mathbb{N}$ and two integral vectors $(k_1, \cdots, k_{m}) \in \mathbb{N}^{m}$ and $(k'_1, \cdots, k'_{m})\in \mathbb{N}^{m}$ that are not equal up to a permutation such that for any $t \in I$,  $ f_{t,m}(k_1, \cdots, k_{m}) = f_{t,m}(k'_1, \cdots, k'_{m})$.}
	\end{lemma}
	\begin{proof}  
		
		Let $q$ be a positive integer, $I$ be a finite subset of $ \mathbb{N}$ and  $(f_{t,m}: \mathbb{R}^{m} \rightarrow \mathbb{R} )_{(t,m)\in I\times \mathbb{N}} $ be a double sequence of piecewise polynomial satisfying the assumptions of the lemma. Let $m$ be any natural number such that $m > q^2 \lvert I \rvert $. For any positive integer $M$, let $F_M$ be the box $\{ (k_1, \cdots, k_m) \in \mathbb{Z}^{m}: \forall i \quad 1 \leq k_i \leq M\}$. 

		There are $  \binom{M+m-1}{m}$ multisets of size $m$ that can be formed from $\{1, \cdots, M\}$. One can find that many  elements of $F_M$ that are not equal up to a permutation.  Let $r$ be an upperbound on the number of pieces of $f_{t,m}$ for every $t$, i.e. $r:=\max \{ \text{number of pieces of } f_{t,m} : t \in I \}$: note that $r$  exists since $I$ is finite, but may depend on $m$. Then, there exists a subset of $F_M$ with at least $\frac{1}{r^{\lvert I \rvert}}\binom{M+m-1}{m}$ integral vectors that are not equal up to a permutation where for every $t$, $f_{t,m}$ is a polynomial $P_{t,m}$ of degree at most $q$, and $P_{t,m}$ is $q$-generated. This is true as the collection of piecewise polynomial functions $(f_{i,m})_{i\in I}$ divide $F_M$ in at most $r^{\lvert I \rvert }$ regions, where in each region, for every $i \in I$, $f_{i,m}$ is polynomial.

		
		Due to Lemma \ref{lemmanbvaluesgeneratedp}, the pigeonhole principle gives us that all polynomials will be equal on some vectors   $(k_1, \cdots, k_m)$ and $(k'_1, \cdots, k'_m)$ of $F_M$ not equal up to a permutation as soon as: \begin{equation}\label{eq:collisioncondition}\underbrace{\frac{1}{r^{\lvert I \rvert }}\binom{M+m-1}{m}}_{\substack{\text{number of multisets formed} \\ \text{ from a region of  $F_M$ where }  \\ \text{  each $f_{t,m} $ is polynomial } }}   \quad >  \underbrace{ \left(  2  \max_{i \in [q], t \in I}( \max_{\alpha \in S_{m,t}} \lvert \lambda^{m,t}_{\alpha, i} \rvert ) M^q \binom{m+q-1}{q} +1\right)^{q\lvert I \rvert }}_{\substack{ \text{number of values $(P_{1,m}, P_{2,m}, \cdots, P_{\lvert I \rvert, m})$} \\ \text{ can take at most on $F_M$}}} \end{equation}
		where the $\lambda^{m,t}_{\alpha, i}$ are  integers in the decomposition of the polynomials $f_{t,m}$ over generators $\sigma^{m}_{1}, \cdots, \sigma^{m}_{q} \in \R$ for the monomial $X^{\alpha}$.  
		
		Such a value of $M$ can be found by noticing that $\binom{M+m-1}{m}$ is a polynomial of $M$ of degree $m$ whereas the right hand side is a polynomial of $M$ of degree $q^2\lvert I \rvert$. Since we chose $ m$ to be greater than  $ q^2 \lvert I \rvert $, there exists $M \in \mathbb{N}$ such that Equation~\eqref{eq:collisioncondition} holds. Hence there exist $k$ and $k'$ whose coordinates are not equal up any permutation and such that $f_{i,m}(k_1, \cdots, k_m)=f_{i,m}(k'_1, \cdots, k'_m)$ for any $i \in I$. 
	\end{proof}

	\begin{lemma}\label{compositiongenerationpol}
		Let $q$ be a positive integer. Let $F: \mathbb{R}^{d_2} \rightarrow \mathbb{R}$ be a piecewise multivariate polynomial whose pieces are all $q$ generated. Let $G : \mathbb{R}^{d_1} \rightarrow \mathbb{R}^{d_2}$ be a function whose coordinates are piecewise polynomial functions whose pieces are all $q$-generated. 
		
		Then the pieces of the piecewise polynomial $F \circ G : \mathbb{R}^{d_1} \rightarrow \mathbb{R}$ are all  $\tilde{q}$-generated, where $\tilde{q}$ depends only on $q$ and the degree of the pieces of $F$ and $G$.    
	\end{lemma}
	\begin{proof}

		We first make the following observation.  Let $A_1, \cdots, A_{r_1} \subseteq \mathbb{R}^{d_2}$ 
		be polynomial regions of $F$ and $B_1, \cdots, B_{r_2} \subseteq \mathbb{R}^{d_1}$ be polynomial regions of $G$ that union to $\mathbb{R}^{d_2}$ and $\mathbb{R}^{d_1}$ respectively (such always exist, cf. Definition \ref{def:piecewise-poly}). 
		For each $i \in [r_1]$, and $j \in [r_2]$, let $Z_{ij} := (B_i \cap G^{-1}(A_j) )$. Then:
		\begin{itemize}
			\item $G$ and $F \circ G$ are polynomial on each  $Z_{ij}$, as $Z_{ij} \subseteq B_i$, and $G(Z_{ij}) \subseteq {A}_{j}$.
			\item The $Z_{ij}$ union to $\mathbb{R}^{d_1}$, as $\bigcup_{i \in [r_2], j\in [r_1]}Z_{ij} = \bigcup_{j \in [r_1]} \bigcup_{i \in [r_2]}  B_i \cap G^{-1}(A_j) = \bigcup_{j \in [r_1]} \left( G^{-1}(A_j) \bigcap \underbrace{\left( \bigcup_{i \in [r_2]}  B_i \right)}_{\mathbb{R}^{d_1}} \right) = \bigcup_{j \in [r_1]} G^{-1}(A_j) = \mathbb{R}^{d_1}$. The last equality follows from  $\bigcup_{i\in[r_1]} A_i = \mathbb{R}^{d_2} \implies \bigcup_{i\in[r_1]} G^{-1}(A_i) = \mathbb{R}^{d_1}$.
		\end{itemize} Hence, a polynomial piece of $F \circ G$ is of the form $F_i \circ G_j $ where $F_1,\cdots, F_{r_1} $ are the polynomial pieces of $F$, and $G_1, \cdots, G_{r_2}$ are the polynomials pieces of $G$.
		
		
		If  $ F_i = \sum_{\alpha \in S_{F_i}} \gamma_\alpha X^{\alpha}$ and $G_j = \sum_{\beta \in S_{G_j}} \nu_{\beta} X^{\beta}$ with $\gamma_\alpha$ and $\beta_\nu$ integers then
		
		\begin{align*}
		(F_i \circ G_j )(X) &= \sum_{\alpha \in S_{F_i}} \gamma_\alpha \left(  \sum_{\beta \in S_{G_j}} \nu_{\beta} X^{\beta}  \right)^{\alpha}  \\
		& = \sum_{\alpha \in S_{F_i}} \left( \sum_{k=1}^{q} \lambda_{k, \alpha} \sigma_{k} \right) \left( \sum_{\beta \in S_{G_j} } \left( \sum_{k=1}^{q} \eta_{k, \alpha} \sigma_k \right) X^{\beta} \right)^{\alpha} 
		\end{align*}
		After expansion, the polynomial $(F_i \, \circ  \, G_j )$  has coefficients that are linear combinations of multivariate monomials of $\sigma_1, \cdots, \sigma_q$ (of total degree at most the degree of $F_j +1$ in this case) with integer coefficients, as the $\lambda$ and $\eta$ coefficients are supposed to be integers. This shows that $F_i \circ G_j$ is $\tilde{q}$-generated, where $\tilde{q}$ depends only on $p$, the degree of $F_i$ and $G_j$.
	\end{proof}

	\begin{lemma}\label{lemma:polyfunction}
		For positive integers $m$, $k_1, \cdots, $ and $k_m$,  let $T[k_1, \cdots, k_m]$ designate the rooted tree illustrated in Figure \ref{figpolytreecounter}. For any vertex $v \in V(T[k_1, \cdots, k_m])$, 
		Let $\xi^{t}(T[k_1, \cdots, k_m],v)$ be the embedding of the vertex $v \in V$ obtained via a GNN with piecewise activation functions after $t$ iterations (where $\xi^{0}(u)=1$ for any vertex $u$ of $T[k_1, \cdots, k_m]$). 
		Then, for any iteration $t$, there exists an integer $q$ such that for any integer $m$, and any vertex $v \in V$, there exists a symmetric multivariate piecewise polynomial function $F_m$ such that  
		\begin{itemize}
			\item $\xi^{t}(T[k_1,\ldots,k_m],v)=F_m(k_1, \cdots, k_m)$. 
			\item 
			$\mathsf{deg}(F_m) \leq q$ 
			\item each piece of $F_m$ is $q$-generated 
		\end{itemize}

	\end{lemma}
	
	\begin{proof}
		We first prove all properties of $\xi^{t}(T[k_1,\ldots,k_m],v)$ by induction on $t$. 
		\medskip
		
		\emph{Base case:} for $t=0$ this is trivial since all vertices are initialised with the constant polynomial 1, whose degree does not depend on $m$, and is finitely generated.
		\medskip
		
		\emph{Induction step:} Suppose the property is true at iteration $t$, i.e there exists an integer $q_t$ such that for each vertex $w$, and for every integer $m$, $\xi^{t}(T[k_1, \ldots, k_m],w)$ is a multivariate polynomial of the $k_i$'s whose degree is upperbounded by  $q_t$, and is $q_t$ generated. For every vertex $v $ of $T[k_1, \cdots, k_m]$
		$$\xi^{t+1}(T[k_1, \ldots, k_m],v) = \phi(\xi^{t}(T[k_1, \ldots, k_m],v), \sum_{w\in N(v)} \xi^{t}(T[k_1, \ldots, k_m],w))$$ where $\phi$ is a piecewise bivariate polynomial (as a Neural Network $\mathbb{R}^{2}\rightarrow \mathbb{R}$ with piecewise polynomial activation). 
		By composition $\xi^{t+1}(T[k_1, \ldots, k_m],v)$ is a piecewise multivariate polynomial of $k_1, \cdots, k_m$ such that for every integer $m$,  
		$\mathsf{deg}(\xi^{t+1}(T[k_1, \linebreak \ldots, k_m],v)) \leq q_t \cdot \mathsf{deg}(\phi)$. 
		Furthermore, since for every $m$ and every vertex $u$, $\xi^{t}(T[k_1, \ldots, k_m],u)$ is supposed to be $q_{t}$ generated, then using the above update rule, Lemma \ref{compositiongenerationpol} gives us that by composition, $\xi^{t+1}(T[k_1, \ldots, k_m],v)$ is $\tilde{q_t}$ generated, where $\tilde{q_t}$ depends only on the degre of $\phi$ and $q_t$.
		
		Setting $q_{t+1} := \max (q_t \cdot \mathsf{deg}(\phi), \, \tilde{q_{t}} )$ gives us the desired properties (for every vertex $v$ and every integer $m$, $\xi^{t+1}(T[k_1,  \cdots, k_m], v)$ has degree at most $q_{t+1}$ and is $q_{t+1}$ generated), and ends the induction on $t$.
	\end{proof}

	\begin{proof}[Proof of Theorem \ref{theorem:uniform}]
		We already know \cite{grohe2021logic}  that color refinement refines any recurrent GNN (even with an architecture of unbounded size). We prove the existence of pairs of graphs that can be separated by the color refinement algorithm, but cannot be separated by a recurrent GNN of fixed (but arbitrary) size. 
		We use $T[k_1, \cdots, k_m]$ to refer to the tree illustrated in Figure \ref{figpolytreecounter}. This tree has depth two, a root node $s$, and contains $m$ nodes at depth one. Each vertex $i$ at depth $1$ has exactly $k_i-1$ ``children'' at depth two (and therefore $k_i$ neighbors, where $k_i$ is a positive integer). In the following, all vertices have color label $1$.

		\emph{Claim: Let $T[k_1, \cdots, k_m]$ and $T'[k'_1, \cdots, k'_m]$ be two rooted trees given by Figure \ref{figpolytreecounter}. If the $k_i$'s and $k'_i$'s are not equal up to a permutation, the color refinement distinguishes $s$ and $s'$ after two iterations, i.e. $\mathsf{cr}^{2}(s) \neq \mathsf{cr}^{2}(s')$. }
		\begin{proof}[Proof of claim]
			Simply note that 
			$$\mathsf{cr}^{2}(s) = ( \mathsf{cr}^{1}(s), \{\{ \mathsf{cr}^{1}(x_1), \cdots,  \mathsf{cr}^{1}(x_m) \}\} ) $$
			$$\text{where} \quad \mathsf{cr}^{1}(s) = (\underbrace{1}_{\mathsf{cr}^{0}(s)}, \{\{ \underbrace{ 1, \cdots, 1}_{m \, \text{times}} \}\} )$$
			$$\text{and} \quad \forall i \in \{1, \cdots, m \} \quad \mathsf{cr}^{1}(x_i) = (\underbrace{1}_{\mathsf{cr}^{0}(x_i)} , \{ \{ \underbrace{1, \cdots, 1 }_{k_i \, \text{times} }\}\})$$
			hence $\mathsf{cr}^{2}(s)$ is uniquely determined by the multiset $\{\{k_1, \cdots, k_m\}\}$.
		\end{proof}

		Let $T > 0$ be a positive integer, and for $0 \leq t \leq T$, let $f_{t,m}(k_1, \cdots, k_m):=\xi^{t}(T[k_1, \ldots, k_m],s)$ be the value returned by a GNN with piecewise polynomial activation after $t$ iterations (note that the embeddings are one-dimensional because only one color is used). Using Lemma \ref{lemma:polyfunction}, there exists an integer $q$ such that the double sequence 
		$( f_{t,m})_{t \in \{0, \cdots, T\}, m \in \mathbb{N}}$ of  piecewise multivariate polynomials has  degree at most $q$ 
		and such that every $f_{t,m}$ is $q$-generated. Lemma \ref{lemma:polycollision} with $I=\{0, \cdots, T\}$ tells us that there exists $m\in\mathbb{N}$, and two vectors $k\in \mathbb{N}^{m}$ and $k'\in\mathbb{N}^{m}$ whose coordinates are not equal up to permutations, such that for any $t\in \{0, \cdots, T\}$, $f_{t,m}(k_1, \cdots, k_m) =f_{t,m}(k'_1, \cdots, k'_m)$. 
	\end{proof}
	
	\begin{remark}Note that in Theorem \ref{theorem:uniform}, depth two is minimal: for  any pair of  non isomorphic rooted trees of depth one, any GNN with one neuron perceptron, an injective activation function, weights set to one, and zero bias can distinguish their root vertex in one iteration. Indeed, in that case, $\xi^{1}(s) = \sigma(1+\mathsf{deg}(s)) $ if the GNN is recurrent with a combine function given by $\phi: \mathbb{R}^{2} \rightarrow \mathbb{R},(x_1, x_2) \mapsto \sigma(x_1 + x_2)$. Hence, $\xi^{1}(s)\neq\xi^{1}(s')$ as soon as $\sigma $ is injective and $s$ and $s'$ have distinct degree.
	\end{remark}
	
	\section{Activations that are not piecewise polynomial }\label{sec:5}
	
	In this Section we present a proof of Theorem \ref{result:other:activations}. We prove that for any pair of non isomorphic rooted trees of depth two, i.e. trees of the form $T[k_1, \cdots, k_m]$ and $T'[k'_1, \cdots, k'_n]$ (here the $k_i$'s and $k'_i$'s are all greater than or equal to $1$, cf. Figure \ref{figpolytreecounter}) can be distinguished by a bounded GNN with any of the following activation functions: exponential, sigmoid, or a hyperbolic sine, cosine or tangent function. Consider the following $1$-neuron perceptron $\phi$ with activation function $\sigma$, $\phi: \mathbb{R}^{2} \rightarrow \mathbb{R},\; \phi(x_1,x_2) = \sigma( x_1  +  x_2)$. Then it is easy to see that:
	\begin{align*}
	\forall v \in V(T[k_1, \cdots, k_m]) \quad \xi^{1}(v) &= \sigma( \xi^{0}(v) + \sum_{w\in N(v)} \xi^{0}(w))  = \sigma(1 + \mathsf{deg}(v)) \label{eq:update1}\\
	\xi^{2}(v) &= \sigma(\sigma(1 + \mathsf{deg}(v)) +  \sum_{w\in N(v)} \sigma(1 + \mathsf{deg}(w)) 
	\end{align*} 
	In particular $\xi^{2}(s) = \sigma( \sigma( 1 + m) + \sum_{i=1}^{m}\sigma(k_i+1))
	$. 
	Now suppose $\sigma$ is either injective on $\mathbb{R}$, or nonnegative and injective on $\mathbb{R}^{+}$ (this is the case for the exponential, the sigmoid, the hyperbolic tan, and the hyperbolic cosine and sine), $s$ and $s'$ are vertices of two trees with potentially different number of leaves $m$ and $n$, then
	\begin{equation}\label{eq:it2sigma}
	\xi^{2}(s) = \xi^{2}(s') \iff  \sum_{i=0}^{m}\sigma(k_i+1)  =\sum_{i=0}^{n}\sigma(k'_i+1)\end{equation}
	where $k_0:=1+m$ and $k'_0:=1+n$. The goal of the remainder of this section is to prove that the right hand side equality of~\eqref{eq:it2sigma} implies $m=n$ and $k_i$'s are the same as $k'_i$'s, up to a permutation, for the activation functions $\sigma$ of Theorem \ref{result:other:activations}. 
	

	\begin{theorem}[Lindemann-Weierstrass Theorem, 1885]\label{lindermannweierstrassth}
		If $\alpha_1, \cdots, \alpha_n$ are distinct algebraic numbers, then the exponentials $e^{\alpha_1}, \cdots,  e^{\alpha_n}$ are linearly independent over the algebraic numbers.
	\end{theorem}
	
	\begin{lemma}\label{lemmasumexp}
		Let $n$ and $m$ be positive integers, and $\alpha_1, \cdots, \alpha_n$ and $\alpha'_1, \cdots, \alpha'_m$ be algebraic numbers. Then $\sum_{i=1}^{n} e^{\alpha_i} = \sum_{i=1}^{m} e^{\alpha'_i}$ if and only if $m=n$ and the $\alpha_i$'s and $\alpha'_i$'s are equal up to a permutation.
	\end{lemma}
	\begin{proof}
		$(\Longleftarrow )$ is clear. For $(\Longrightarrow )$, by contradiction suppose that the $\alpha_i$'s and $\alpha'_i$'s are not equal up to a permutation. First, if the $\alpha_i$'s (resp. $\alpha'_i$'s) are not distinct one can group them by their number of occurrences in both sums. Then, we would have a linear dependence with integer coefficients of exponentials of integers. This contradicts Theorem \ref{lindermannweierstrassth} (Linderman-Weirstrass).
	\end{proof}

	\begin{proof}[Proof of Theorem \ref{result:other:activations}]
		Without loss of generality, suppose the $k_i$'s and $k'_i$'s are ordered in increasing order. For ease of notation, let $\alpha  $ and $\alpha'$ be the vectors defined as $\alpha_i = k_i+1$ for all $i\in \{1, \cdots, m\}$ and $\alpha'_i=k'_i+1$ for all $i\in \{1, \cdots, n\}$. We will now prove that~\eqref{eq:it2sigma} implies $\alpha=\alpha'$ in each case.

		- $\sigma \in \{\mathsf{sigmoid}, \mathsf{tanh}\}$. 
		In the case of the sigmoid,~\eqref{eq:it2sigma} yields the following equation after multiplication by the the product of the denominators:
		\begin{equation*}\resizebox{\hsize}{!}{$
			\left( \sum_{i=1}^{m}  e^{\alpha_i} \left( \prod_{\substack{j=1 \\ j\neq i }}^{m} (1 +e^{\alpha_j}) \right) \right ) \prod_{i=1}^{n} (1+ e^{\alpha'_i})= \left( \sum_{i=1}^{n}  e^{\alpha'_i} \left( \prod_{\substack{j=1 \\ j\neq i }}^{n} (1+e^{\alpha'_j}) \right) \right)  \prod_{i=1}^{m} (1+e^{\alpha_i})
			$}
		\end{equation*}
		After developing and grouping each hand side into linear combinations of exponentials we obtain an equation of the form: 
		\begin{equation}\label{eq:notcompact}
		\sum_{\substack{S \subseteq \{1, \cdots, m\}  \\ T \subseteq \{ 1, \cdots, n\} }} \gamma_{S,T}  \exp(\alpha_S + \alpha'_T) =\sum_{\substack{S \subseteq \{1, \cdots, m\}  \\ T \subseteq \{ 1, \cdots, n\} }} \gamma_{S,T}  \exp(\alpha'_S + \alpha_T)
		\end{equation}
		where for $S \subseteq \{1, \cdots, m\}$, $\alpha_S:=\sum_{i \in X }\alpha_i$ (resp. for $T\subseteq \{1, \cdots, n\}, \alpha'_T:=\sum_{i \in X }\alpha'_i$). {All $\gamma_{S,T}$ are integers, hence algebraic} and $\gamma_{\emptyset, T}=0$ for all subsets $T \subseteq \{1, \ldots, n\}$. 
		
		
		
		{ We will prove by strong induction on $\max(m,n)$ that  $\sum_{i=1}^{m} \sigma(\alpha_i) = \sum_{i=1}^{n}\sigma(\alpha'_i) \implies m=n $ and $\alpha = \alpha'$. 
			
			\medskip
			
			\emph{Base case:}
			If $\max(m,n)=1$ and $\sum_{i=1}^{m}\sigma(\alpha_i)=\sum_{i=1}^{n}\sigma(\alpha_1)$, then, either $m=n=0$, or $m=n=1$. In the first case, this is vacuously true. In the second case, we have that $\sigma(\alpha_1)=\sigma(\alpha'_1)$, and then  $\alpha_{1}=\alpha'_{1}$ follows from the injectivity of the sigmoid. 
			
			\medskip
			
			\emph{Induction step:}
			We suppose that for some given positive integer $p$, and any nonnegative integers $\alpha_1, \cdots, \alpha_m$ and $\alpha'_1, \cdots, \alpha'_n$, such that $\max(m,n) \leq p $,  then  $\sum_{i=1}^{m}\sigma(\alpha_i)=\sum_{i=1}^{n}\sigma(\alpha'_i) \implies m=n $ and $ \alpha=\alpha'$.

			Let $\alpha$ and $\alpha'$ be vector of integers of size $m$ and $n$ such that $\max(m,n) =p+1$, and $\sum_{i=1}^{m}\sigma(\alpha_i)=\sum_{i=1}^{n}\sigma(\alpha'_i)$. We saw that~\eqref{eq:notcompact} can be derived from this equality, where $\gamma_{S, T}$ are algebraic numbers satisfying $\gamma_{\emptyset,T} = 0$ for all $T\subseteq \{1, \ldots, n\}$. Moreover, the coordinates of $\alpha$ and $\alpha'$ are ordered, hence the smallest term on the left hand side is $\exp(\alpha_1)$ and the smallest term on the right hand side is $\exp(\alpha'_1)$. Using Lemma \ref{lemmasumexp}, this implies that $\alpha_1 = \alpha'_1$. Therefore, $\sigma(\alpha_1) = \sigma(\alpha'_1)$. In turn, this implies $\sum_{i=2}^{m}\sigma(\alpha_i) =\sum_{i=2}^{m}\sigma(\alpha'_i) $. We can apply the induction hypothesis on the vectors $(\alpha_2, \cdots, \alpha_{m})$ and $(\alpha_2, \cdots, \alpha'_{n})$ which both have size $\leq p$. Hence, we obtain that  $m-1=n-1$ and $(\alpha_{2}, \cdots, \alpha_{m})=(\alpha'_{2}, \cdots, \alpha'_{m})$. This in turn proves that $m=n$, and $\alpha=\alpha'$, which ends the induction.
		}
		
		If $\sigma=\mathsf{tanh}=\frac{\exp(2\cdot)+1}{\exp(2\cdot)-1}$. After multiplication by the product of the denominators,~\eqref{eq:it2sigma} yields:
		\begin{equation*}\resizebox{\hsize}{!}{$
			\left(  \sum_{i=1}^{n} (e^{2\alpha_i}-1 )\prod_{j=1, j \neq i}^{n}(e^{2\alpha_j} +1) \right) \prod_{j=1 }^{m} (1+e^{2\alpha'_j}) =  \left( \sum_{i=1}^{m} (e^{2\alpha'_i}-1) \prod_{j=1, j \neq i}^{m}(e^{2\alpha'_j} +1) \right) \prod_{j=1 }^{n}  (1+e^{2\alpha_j})$}
		\end{equation*}
		After developing into a linear combination of exponentials on each side, the arguments containing $\alpha_T$ with $T\neq \emptyset$ on the left hand side and $\alpha'_T$ with $T\neq \emptyset$ on the right hand side have positive algebraic coefficients. There are also arguments of the form $\alpha'_T$ on the left hand side and $\alpha_T$ on the right hand side (in other words, $\gamma_{\emptyset, T}\neq 0$, unlike the sigmoid case). However, note that the coefficients corresponding to these terms are (algebraic and) negative. Hence, as a consequence of Lemma \ref{lemmasumexp}, the arguments with negative coefficients in front of the exponentials must match up on each side, and we are left with an equation similar to~\eqref{eq:notcompact} (the arguments have a factor $2$), where again $\gamma_{\emptyset, T} = 0 $. We can apply the same reasoning by induction as for the sigmoid case, to prove that $\alpha=\alpha'$.

		
		- $\sigma \in \{\mathsf{sinh}, \mathsf{cosh} \}$. If $\sigma = \mathsf{cosh}$, then~\eqref{eq:it2sigma} yields:
		\begin{equation*} 
		\left( \sum_{j=1}^{n}  \exp(\alpha_j) - \sum_{j=1}^{m}\exp(\alpha'_j) \right) + \left( \sum_{j=1}^{n}  \exp(-\alpha'_j) - \sum_{j=1}^{m}\exp(-\alpha_j) \right) =0
		\end{equation*}
		Due to Lemma \ref{lemmasumexp}, this can only happen if $m=n$ and for all $ j \in \{1, \cdots, n\}$, $\alpha_j = \alpha'_j$, because $\alpha_j$,  $\alpha'_j$ are algebraic for any $j \in \{1, \cdots, n\}$, and the $\alpha_j$'s  and $\alpha'_j$'s are ordered and positive. We conclude that $\alpha=\alpha'$. The case $\sigma \in \{ \mathsf{sinh} \}$ can be treated similarly. 
	\end{proof}
	
	\section{Acknowledgements}
	
	The authors gratefully acknowledge support from Air Force Office of Scientific Research (AFOSR) grant
	FA95502010341 and National Science Foundation (NSF) grant CCF2006587. The authors are also very grateful to insightful comments from reviewers that helped to improve the paper significantly. In particular, we learnt of the reference~\cite{aamand2022exponentially} from one of the reviewers.  We also thank Eran Rosenbluth for pointing out an error in our first proof of Theorem \ref{theorem:uniform}. 

	\bibliographystyle{alpha}
	\bibliography{sample}

\end{document}